\documentclass{article} 
\usepackage{iclr2027_conference,times}

\usepackage{amsmath,amsfonts,bm}

\def\eqref#1{equation~\ref{#1}}

\def\1{\bm{1}}

\DeclareMathAlphabet{\mathsfit}{\encodingdefault}{\sfdefault}{m}{sl}
\SetMathAlphabet{\mathsfit}{bold}{\encodingdefault}{\sfdefault}{bx}{n}

\newcommand{\E}{\mathbb{E}}

\newcommand{\Cov}{\mathrm{Cov}}

\usepackage{hyperref}
\usepackage{url}
\usepackage[linesnumbered,ruled,vlined]{algorithm2e}
\SetKwComment{Comment}{/* }{ */}
\SetNoFillComment

\SetCommentSty{proxypropcommentfont}
\usepackage{graphicx}
\usepackage{amsmath,amsfonts,amssymb}
\usepackage{pgfplots} 
\usepgfplotslibrary{groupplots}
\usepgfplotslibrary{fillbetween}
\usetikzlibrary{positioning, arrows.meta}
\usetikzlibrary{calc,patterns}
\usetikzlibrary{shapes.geometric}
\usepackage{subfig}
\usepackage{hhline}
\usepackage{amsthm}
\usepackage{comment}
\newtheorem{theorem}{Theorem}[section]
\newtheorem{corollary}[theorem]{Corollary}
\newtheorem{remark}[theorem]{Remark}
\newtheorem{proposition}[theorem]{Proposition}
\usepackage{wrapfig} 

\newcommand{\algo}{ProxyProp}

\pgfplotsset{every axis title/.append style={text depth=0.25ex}}
\newsavebox{\legendbox}
\newlength{\panelAwidth}
\newlength{\panelBwidth}

\definecolor{ide}{RGB}{90,90,90}
\definecolor{cReuse}{HTML}{CC3311}  
\definecolor{cLinear}{HTML}{D85D44}  
\definecolor{cIdentity}{HTML}{E37E68}  
\definecolor{cRandom}{HTML}{ED9D8B}  
\definecolor{cNores}{HTML}{FF7F0E}  
\definecolor{cShared}{HTML}{8B4486}  
\definecolor{cDFA}{HTML}{2A78D6}  
\definecolor{cDRTP}{HTML}{009988}  
\definecolor{cDNI}{HTML}{332288}  
\definecolor{cDBlock}{HTML}{117733}  
\definecolor{cBP}{HTML}{444444}  
\definecolor{cMeZO}{HTML}{20B2AA} 

\usepackage{xcolor}
\pgfplotsset{compat=1.18}

\title{Local proxies for memory-efficient training}

\author{Wu S. Zihan\thanks{Equal contribution from WS Zihan and BA Yuan.} \\
School of Computer and \\ Communication Sciences \\
EPFL \\
Lausanne, Switzerland \\
\small{\texttt{zihan.wu@epfl.ch}} \\
\And
Brian A. Yuan\footnotemark[1]\\
Machine Learning Research \\
TU Wien \\
Vienna, Austria \\
\small{\texttt{brian.yuan@tuwien.ac.at}}
\And
Guillaume Bellec \\
Machine Learning Research \\
TU Wien \\
Vienna, Austria \\
\small{\texttt{guillaume.bellec}} \\ \small{\texttt{@tuwien.ac.at}}
}

\iclrfinalcopy 

\begin{document}

\maketitle


\begin{abstract}
Local learning algorithms can reshape resource requirements of artificial intelligence.
The memory bottleneck of standard end-to-end backpropagation lies in the need to store all network activations until the loss gradient has been propagated back through.
Local learning avoids this, but typically at a cost in accuracy. 
We study how to achieve accurate and memory efficient local learning in transformers. 
For each transformer block, we train a local proxy network to predict the global network output via distillation. 
Gradients can be independently propagated through this local proxy while estimating well the backpropagation gradients.
In theory, this approach is exact for deep linear networks; in non-linear residual models, the gradient error can be bounded when using the appropriate residual proxy.
We verify this numerically and demonstrate that optimizing the distillation loss improves gradient estimation.
Evaluating the performance of proxy propagation on CIFAR-100 in vision transformers, we find that it achieves a better accuracy/memory trade-off compared to other local algorithms. 
Finally, we show that the method is further relevant for fine-tuning large language models on a small memory budget. 
Our method trains faster and achieves higher performance than memory-efficient zeroth-order learning algorithms with comparable memory cost.

\end{abstract}

\section{Introduction}
\definecolor{cReuse}{HTML}{CC3311}  
\definecolor{cLinear}{HTML}{D85D44}  
\definecolor{cIdentity}{HTML}{E37E68}  
\definecolor{cRandom}{HTML}{ED9D8B}  
\definecolor{cNores}{HTML}{FF7F0E}  
\definecolor{cShared}{HTML}{8B4486}  
\definecolor{cDFA}{HTML}{2A78D6}  
\definecolor{cDRTP}{HTML}{009988}  
\definecolor{cDNI}{HTML}{332288}  
\definecolor{cDBlock}{HTML}{117733}  
\definecolor{cBP}{HTML}{444444}  
\definecolor{cMeZO}{HTML}{20B2AA} 

Conventionally, intelligence emerges in neural network models from end-to-end training with backpropagation \citep{rumelhart1986learning}.
By design, backpropagation induces a backward locking problem \citep{jaderberg2017syngrad}. 
A forward pass first computes the network activations from inputs to outputs and then the backward pass propagates gradients in reverse order.
In the meantime, network activations need to be stored, increasing memory requirements with network depth. 
The large memory consumption is a major manufacturing cost and increases the energy consumption of training datacenters~\citep{gholami2024ai}. 
When training large language models (LLMs) under a limited memory budget, the network depth and sequence length that fit in memory therefore constrain the range of tasks and the reasoning capability of the model.

Historically inspired by neuroscience research \citep{hebb1949organization,markram2011history, gerstner2018eligibility}, local learning is one way of avoiding the backward locking problem. 
Since the brain appears to be a distributed network without a precise backward gradient communication pathway, it is natural to suppose that biological sub-networks mainly learn from their immediate inputs and internal representations \citep{lillicrap2020backpropagation}.
Yet few local learning algorithms have been scaled to transformers, and attempts so far have had limited success~\citep{filipovich2022scaling}.
Most of these algorithms were not designed to approximate the backpropagated gradient, so they struggle in deep architectures; on top of this, those designed for biological plausibility often come with an increase in training time. 
Within the scope of local learning, we seek an alternative to backpropagation that is mathematically grounded and memory efficient without compromising training time.


Towards this goal, we define locality by partitioning a large transformer into blocks, with the gradient computation in each block decoupled in the sense of~\cite{jaderberg2017syngrad}: block parameters are updated immediately after the forward pass of the block. 
Hence, we say that learning is ``local'' if it needs to store activations from only one block at a time in memory.
Additionally, when the block updates are fully independent, we say that the training algorithm is ``unlocked'', making possible an efficient form of model parallelism~\citep{belilovsky2020dgl}.
Here, we follow a principled approach to estimate global gradients locally.
To each block we equip a local proxy network with the objective of predicting the global network output via distillation. Gradients of each block are then computed with local backpropagation through the proxy instead of the downstream blocks.
This way, proxy propagation (ProxyProp) offers theoretical insights that translate into a practical algorithm:
\begin{itemize}

    \item {\em Theory:} ProxyProp is an approximation of the global backpropagation. Despite the locality and the simplicity of the prediction objective, a linear proxy implicitly estimates the expected network Jacobian, and is therefore exact in deep linear networks.
    \item {\em Theory:} For non-linear residual networks, when the proxy follows a residual parameterization, the local gradients are exact at initialization and can achieve a bounded gradient estimation error when the skip connections of the network are dominating. 
    
    \item {\em Simulation:} For supervised classification in vision transformers (ViT), across various network partitionings, ProxyProp offers a better trade-off between memory and accuracy over other local algorithms. 
    \item {\em Simulation:} ProxyProp enables memory-efficient fine-tuning of LLMs, achieving higher performance than zeroth-order optimization techniques (MeZO) at comparable memory cost (even when MeZO uses 8x more compute time).
    %
\end{itemize}

\begin{figure}
      \centering
      \subfloat[Algorithm schematic]{%
          \resizebox{0.48\linewidth}{!}{\begingroup   
\definecolor{cGrad}{HTML}{2E8B57}
\begin{tikzpicture}[
  font=\small,
  blk/.style={draw=gray!70, fill=gray!8, rounded corners=2pt, minimum width=1.35cm, minimum height=0.56cm},
  head/.style={blk, fill=gray!22, minimum width=1.55cm},
  dots/.style={fill=white, inner sep=1pt, text=gray!80, font=\small},
  modbox/.style={rounded corners=2pt, minimum width=0.95cm, minimum height=0.52cm},
  proxy/.style={modbox, draw=cReuse, fill=cReuse!22},
  mem/.style={draw=cBP, minimum width=0.17cm, minimum height=0.44cm, inner sep=0pt},
  held/.style={mem, fill=cBP},
  always/.style={mem, fill=cBP!25},   
  stream/.style={gray!80, line width=1.1pt},
  fwd/.style={-stealth, gray!80, line width=0.7pt},
  grad/.style={-stealth, cGrad, line width=1.5pt},
  zlab/.style={font=\scriptsize, text=gray!90, fill=white, inner sep=1.5pt},
  title/.style={font=\small\bfseries},
]
\newcommand{\col}[1]{
  \draw[stream] (0, -0.85) -- (0, 4.92);
  \node[head] (#1h) at (0, 5.2) {head $\;\mathcal{L}$};
  \foreach \y/\n/\f in {0.7/3/{f^{l}}, 2.7/4/{f^{L-1}}, 3.8/5/{f^{L}}} {
    \node[blk] (#1\n) at (0, \y) {$\f$};
    \draw[fwd] (0, \y - 0.5) -- (#1\n.south);
  }
  \draw[fwd] (0, 4.7) -- (#1h.south);
  \node[dots] at (0, -0.6) {$\vdots$};
  \node[dots] at (0, 1.95) {$\vdots$};
  \node[zlab] at (0, 0.0) {$z_{l-1}$};
  \node[zlab] at (0, 1.4) {$z_l$};
  \node[zlab] at (0, 4.5) {$z_L$};
}
\newcommand{\localmem}{
  \node[held] at (-1.25, 0.7) {};
  \node[mem]  at (-1.25, 2.7) {};
  \node[mem]  at (-1.25, 3.8) {};
}
\newcommand{\hatchbar}[2]{
  \node[always] (hb) at (#1, #2) {};
  \begin{scope}
    \clip (hb.south west) rectangle (hb.north east);
    \foreach \d in {-0.4,-0.3,...,0.5} \draw[cBP, line width=0.4pt] (#1-0.2+\d, #2-0.3) -- (#1+0.1+\d, #2+0.3);
  \end{scope}
}
\newcommand{\module}[3]{
  \node[#2, anchor=west] (#1m) at ([xshift=0.75cm]#13.east) {#3};
  \draw[grad] ([yshift=-0.12cm]#1m.west) -- ([yshift=-0.12cm]#13.east);
}

\col{b}
\node[title, text=cBP] at ([yshift=0.5cm]bh.north) {backprop};
\foreach \y in {0.7,2.7,3.8} \node[held] at (-1.25, \y) {};
\draw[grad] (1.35, 4.92) -- (1.35, 0.45);

\begin{scope}[xshift=4.2cm]
\col{d}
\node[title, text=cReuse] at ([yshift=0.5cm]dh.north) {ProxyProp};
\localmem
\module{d}{proxy, text=cReuse}{$p_l$}
\draw[-stealth, cReuse, line width=0.7pt] ([yshift=0.12cm]d3.east) -- ([yshift=0.12cm]dm.west);
\node[font=\scriptsize, text=cReuse, anchor=west] at ([xshift=0.06cm]dm.east) {$\approx z_L$};
\draw[grad] ([yshift=0.5cm]dm.north) -- (dm.north);
\node[font=\scriptsize, text=cGrad, anchor=west, inner sep=1pt] at ([xshift=0.05cm, yshift=0.25cm]dm.north) {$\partial\mathcal{L}/\partial z_L$};
\end{scope}
\end{tikzpicture}
\endgroup}%
          \label{subfig1a:schem}
      }\hfill
      \subfloat[VRAM vs. accuracy trade-off]{%
          \resizebox{0.48\linewidth}{!}{\begin{tikzpicture}
\input{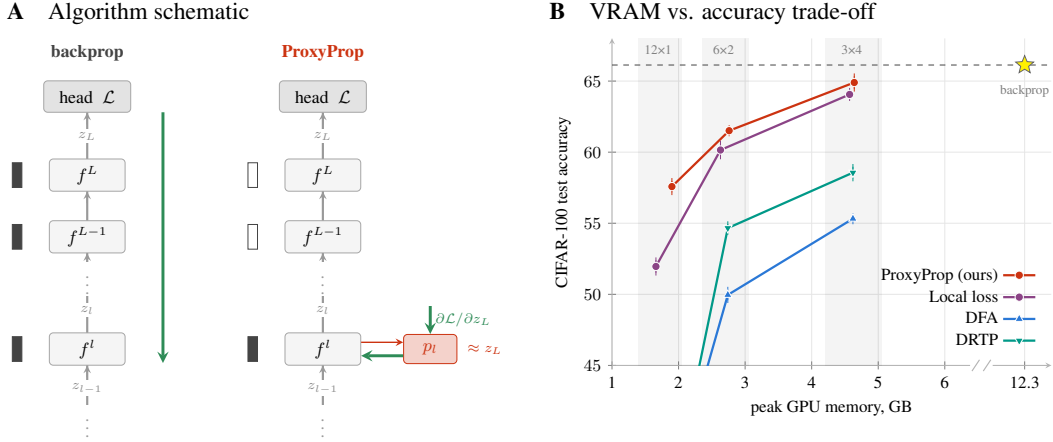}  
\begin{axis}[
  width=0.7\linewidth, height=0.55\linewidth,
  xtick={1,2,3,4,5,6,7.2}, xticklabels={1,2,3,4,5,6,12.3}, minor xtick={},
  xmin=1, xmax=7.6, ymin=45, ymax=68,
  xlabel={peak GPU memory, GB },
  ylabel={CIFAR-100 test accuracy},
  axis lines=left, axis line style={gray!60},
  grid=major, grid style={gray!20},
  tick label style={font=\small}, label style={font=\small},
  legend style={at={(0.98,0.03)}, anchor=south east, draw=none, fill=none, font=\footnotesize},
  legend cell align=right,
  legend plot pos=right,
  error bars/y dir=both, error bars/y explicit,
  error bars/error bar style={line width=0.5pt}, error bars/error mark options={mark size=1.5pt, line width=0.5pt},
  every axis plot/.append style={line width=1.2pt, mark=*, mark size=2.2pt, mark options={fill=., draw=white, line width=0.5pt}},
]
\fill[black, opacity=0.045] (axis cs:1.40,10) rectangle (axis cs:2.05,75);
\fill[black, opacity=0.045] (axis cs:2.35,10) rectangle (axis cs:3.05,75);
\fill[black, opacity=0.045] (axis cs:4.20,10) rectangle (axis cs:5.05,75);
\node[font=\scriptsize, gray, anchor=north, align=center] at (axis cs:1.694,68.) {12×1};
\node[font=\scriptsize, gray, anchor=north, align=center] at (axis cs:2.677,68.) {6×2};
\node[font=\scriptsize, gray, anchor=north, align=center] at (axis cs:4.605,68.) {3×4};
\coordinate (brk) at (axis cs:6.55,45);   
\addplot[cReuse, mark=*]  coordinates {(4.64,64.90)+-(0,.44) (2.76,61.51)+-(0,.19) (1.90,57.58)+-(0,.40)};
\addlegendentry{\algo{} (ours)}
\addplot[cShared, mark=*]  coordinates {(4.57,64.06)+-(0,.26) (2.63,60.15)+-(0,.45) (1.66,51.96)+-(0,.44)};
\addlegendentry{Local loss}
\addplot[cDFA, mark=triangle*, mark size=2.6pt]    coordinates {(4.62,55.32)+-(0,.19) (2.74,49.96)+-(0,.37) (1.87,35.66)+-(0,.24)};
\addlegendentry{DFA}
\addplot[cDRTP, mark=triangle*, mark size=2.6pt, mark options={rotate=180, fill=., draw=white, line width=0.5pt}]   coordinates {(4.62,58.56)+-(0,.41) (2.74,54.66)+-(0,.28) (1.87,34.97)+-(0,.66)};
\addlegendentry{DRTP}
\draw[cBP, dashed, line width=0.5pt] (axis cs:1.0,66.13) -- (axis cs:7.6,66.13);
\node[star, star points=5, star point ratio=2.25,
      fill=yellow, draw=black!60, line width=0.4pt,
      minimum size=10pt, inner sep=0pt]
  at (axis cs:7.2,66.13) {};
\node[font=\scriptsize, gray, anchor=north] at (axis cs:7.2,65.2) {backprop};
\end{axis}
\draw[white, line width=5pt] ($(brk)+(-6pt,0)$) -- ($(brk)+(6pt,0)$);
\draw[gray!60, line width=0.6pt] ($(brk)+(-5pt,-3pt)$) -- ($(brk)+(-1pt,3pt)$);
\draw[gray!60, line width=0.6pt] ($(brk)+(1pt,-3pt)$) -- ($(brk)+(5pt,3pt)$);
\end{tikzpicture}}%
          \label{subfig1b:accmem}
      }
      \caption{\textbf{ProxyProp and its accuracy/memory trade-off.} 
      \textbf{\ref{subfig1a:schem}}: In ProxyProp, when updating block $f^\ell$, only activations of that block need to be stored. The loss gradients are propagated through a proxy function instead of the other blocks.
      \textbf{\ref{subfig1b:accmem}}: Memory and accuracy tradeoff for ProxyProp and other local methods; the 12-layer vision transformer (ViT) is either partitioned into 3, 6 or 12 blocks. 
      }
      \label{fig:placeholder}
\end{figure}

\section{Review: State of the Art}

The ``local learning'' terminology has been defined with different aims in different contexts (see \cite{bredenberg2023formalizing} for a formal review). 
In bio-inspired computing, it is a constraint for bio-plausible learning circuitry.
This line of research originated from the observation that backpropagation requires a backward circuit with symmetric feedback weights, for which there is no anatomical evidence \citep{lillicrap2020backpropagation}. 
Different lines of work have relaxed this requirement \citep{lillicrap2016fa, akrout2019wm, max2024contwm, nokland2016direct}, or sought compatibility with Hebbian learning or predictive coding \citep{whittington2017predictivecoding, scellier2017equilibrium, halvagal2023lpl, illing_2021}. 
All these algorithms formalize how to replace backpropagation through the network hierarchy in space, which is complementary to the replacement of the backpropagation of information through time \citep{bellec2020eprop, murray2019local}. 
Together, these algorithms were foundational in bringing local learning theory to deep networks and have inspired circuitry for learning in physical systems~\citep{momeni2023phyll}, photonics~\citep{wang2026odfa}, analog hardware \citep{pehle2022brainscales} and digital ASIC technology~\citep{su2025elfcore}.

We focus in particular on decoupling the stack of network layers into $L$ gradient-isolated blocks, where each block learns locally.
This can be achieved by stopping the gradient propagation across blocks and optimizing each network block with a local loss, either being supervised \\\citep{belilovsky2019greedy, nokland2019local, frenkel2021learning, shing2026diffusionblocks} or self-supervised \citep{lowe2019gim, illing_2021, hinton2022forward, halvagal2023lpl, siddiqui2024blockwise,chen2025self}. 
Such methods have also been applied to finetuning pretrained models \citep{key2023local, gelbard2026unlocking}. Although local losses work relatively well in practice, it is not understood how the stack of decoupled blocks builds a hierarchy of representations (see ~\cite{delrocq2026self,korchinski2026learn} for some analyses).

Towards a principled algorithm to approximate the global gradient from local information, we build instead a gradient estimate $g_{\theta_\ell}^{loc}$ for layer $\ell$ parameters $\theta_\ell$ that intends to approximate the backpropagation gradient ${\partial \mathcal{L}}/{\partial \theta_\ell}$ denoted $g_{\theta_\ell}^{BP}$. 
The difficulty is that $g_{\theta_\ell}^{loc}$ needs to estimate a global quantity from local network activations.
There are only few local theories explaining the emergence of globally coherent gradient from local calculations.
For a class of self-supervised contrastive losses \citep{zihan2026localssl}, the local estimations can be exact, shown to work well to train unsupervised convolutional networks.
Alternatively, \cite{jaderberg2017syngrad} train a local auxiliary network to explicitly predict the gradient bootstrapped from the following block.
However, in practice, the compounding errors of the auxiliary networks make it numerically challenging. 



\section{Theory: proxy distillation with accurate gradient propagation}

\textbf{Definition of ProxyProp.}
We partition a deep network into $L$ blocks such that: $z_\ell = f_{\theta}^{\ell}(z_{\ell-1})$ where $z_{\ell} \in \mathbb{R}^{d_\ell}$ is the activation of block $\ell$ and $z_0$ is the input. 
We consider an arbitrary loss head $\mathcal{L}(\theta_h, z_L, y)$ that typically includes head parameters $\theta_h$ and may depend on target tokens or labels $y$. 
For brevity, we write $\mathcal{L}(z_L)$.
The aim is to approximate the backpropagated gradients $g_{\theta_\ell}^{BP}$.

Our method defines a proxy network ${p}_{\phi}^\ell$ distilled from the network output $z_L$ such that $z_L \approx {p}_{\phi}^\ell(z_\ell)$. 
For memory efficiency, we construct the proxy such that ${p}_{\phi}^\ell(z_\ell)$ or its gradient propagation (\textit{i.e.} ${\partial {p}_{\phi}^\ell / \partial z_\ell}$) does not require large memory storage.
Typically, ${p}_{\phi}^\ell(z_\ell)$ can be either a linear function of $z_\ell$ or a tiny transformer.
Under this construction summarized in Figure \ref{fig:placeholder}, our algorithm ProxyProp backpropagates the loss gradients through the proxy instead of the expensive layers ($f_\theta^{l+1}, \cdots ,f_\theta^L$), such that our estimator of the full backpropagation gradients can be written as:
\begin{equation}
\label{eq:proxygrad}
    g_{\theta}^{loc} =
    \frac{\partial f_\theta^\ell}{\partial \theta}^\top 
    \frac{\partial p_\phi^\ell}{\partial z_\ell}^\top
    \frac{\partial \mathcal{L}}{\partial z_L} ~~~ \text{(ProxyProp)}
\end{equation}
This equation is the core of ProxyProp (Algorithm~\ref{alg:training}). 
The algorithm is ``local'': after the calculation of $f_\theta^\ell$ (line 7 of Algorithm~\ref{alg:training}), $g_{\theta_\ell}^{loc}$ can be computed and the parameters of block $\ell$ can be updated at once, without waiting for the forward or backward pass of the later blocks.
Consequently, we can release the stored activation $z_{\ell-1}$ (line 11) immediately after, such that the memory requirement for network activations does not scale with the network depth $L$ as it would in backpropagation. 
Table \ref{tab:complexity} in appendix summarizes this: while the memory requirement scales linearly with depth in backpropagation, it is independent of the depth in ProxyProp.

We make three further clarifications \emph{(a-c)} relevant to understanding the design of this algorithm.
\emph{(a)}~Algorithm~\ref{alg:training} is ``local'', and the weight updates are ``unlocked'' such that (line 6 to 11) could be parallelized.
\emph{(b)}~The right hand term ${\partial \mathcal{L}}/{\partial z_L}$ requires a \emph{pre-}forward pass (line 1 to 5) to compute the last network activation $z_L$ and the loss gradient. Thus, there remains a weaker ``locking'' at this level. In the discussion we describe a ProxyProp-fully-unlocked variant (Algorithm ~\ref{alg:unlock}) which can be fully distributed.
\emph{(c)} In this paper, we study ProxyProp as described in Algorithm~\ref{alg:training}, our theoretical analyses relying on the observation that $g_{\theta}^{loc}$ becomes exact if the proxy Jacobian equals the large network Jacobian: $\partial p_{\phi}^\ell/{\partial z_\ell} = {\partial z_L}/{\partial z_\ell}$. 


\begin{algorithm}[t]
\caption{Description of the ProxyProp algorithm. No activations or gradients are stored or computed unless specified. The input $z_0$ is given. $p_{\phi}^L(z_l)$ is always the identity.}
\label{alg:training}
\vspace{4pt}
\Comment{Pre-forward pass and loss head update}
\For{$\ell = 1$ \KwTo $L$}{
    $z_\ell \gets f_\ell(z_{\ell-1})$\;
    \textbf{delete} $z_{\ell-1}$\;
}
Store $z_L$ and ${\partial \mathcal{L}}/{\partial z_L}$\;
Update $\theta_h$ with ${\partial \mathcal{L}}/{\partial \theta_h}$ \;
\vspace{4pt}
\Comment{Local ProxyProp algorithm}
\For{$\ell = 1$ \KwTo $L$}{
    $z_\ell \gets f_\ell(z_{\ell-1})$\;
    Compute the proxy $p_{\phi}^\ell(z_\ell)$\;
    Update $\theta_\ell$ with local gradient from Equation~\ref{eq:proxygrad}\;
    Update $\phi$ with gradient step on the distillation loss from Equation~\ref{eq:proxyfunc}\;
    \textbf{delete} $z_{\ell-1}$\;
}
\end{algorithm}

\textbf{Proxy distillation objective.}
Notoriously in the synthetic gradient literature \citep{jaderberg2017syngrad}, the difficulty in efficient decoupled neural interface (DNI) lies in approximating well the backpropagated gradients at each block. 
Our strategy distills the large network into the proxy $p_\phi$, minimizing:
%
\begin{equation}
    \phi^* = \operatorname{argmin}_{\phi} \mathbb{E}_{z_l}
    \left[|| z_{L} - p_{\phi} ( z_\ell ) ||^2 \right]
    \label{eq:proxyfunc}
\end{equation}
where the expected is estimated by batches of samples.
Distillation from the large network output $z_L$ does not guarantee \textit{a priori} that the backpropagated gradient or the Jacobians $\partial p_{\phi}^\ell/{\partial z_\ell}$ and ${\partial z_L}/{\partial z_\ell}$ will become equal, which is necessary to achieve accurate gradients according to clarification \textit{c}.
Alternatives include a bootstrapped synthetic gradient loss~\citep{jaderberg2017syngrad} and explicit Jacobian matching \citep{czarnecki2017sobolev}; in our experiments in Section~\ref{sec:results}, these either gave inaccurate gradients or required second-order derivatives, resulting in substantial memory consumption.
By extension of the results from \citet{srinivas2018knowledge}, if Equation \ref{eq:proxyfunc} is minimized under random perturbations of $z_\ell$, then so, implicitly, is the squared error of the Jacobians.
In practice, we even observed that minimizing Equation \ref{eq:proxyfunc} aligns the backpropagated gradients without repeated perturbations. 
Below, we suggest an alternative theory that is well supported by our simulations.


\textbf{Gradient consistency in linear proxies.}
We may first consider a proxy function that is linear. 
This case allows for a theoretical treatment explaining cases where optimizing Equation \ref{eq:proxyfunc} yields accurate gradients.
We consider a linear proxy function $p_{\phi}(z_\ell)=P_\ell z_\ell + b_\ell$ with parameters $\phi=(P_\ell,b_\ell)$ that predict the output layer $z_L$ using the hidden layer $z_\ell$.
For each iteration and parameter $\theta$, we can resolve in closed form the parameters $\phi^\star=(P_\ell^*, b_\ell^*)$ (minimizing Equation \ref{eq:proxyfunc}) and analyze the resulting gradients.
In this setting, we find a theoretical link between the Jacobian of the proxy ${\partial p_{\phi}^\ell}/{\partial z_\ell} = P_l^*$ and the unknown Jacobian ${\partial z_L}/{\partial z_\ell}$ of the large network:
\begin{proposition}
\label{thm-exact}
For any (non-linear) network distilled with a linear proxy $p_{\phi}(z_\ell)=P_\ell z_\ell + b_\ell~$, if $z_\ell$ follows a Gaussian distribution $\mathcal{N}(\mu_{\ell}, \Sigma_{\ell})$ with $\Sigma_{\ell}$ positive definite, and the layers $f_{\ell'}$ ($\ell' > \ell$) are locally Lipschitz such that $\E\big[\big\|\partial z_L/\partial z_\ell\big\|\big] < \infty$, the matrix $P_\ell^*$ of the optimal linear proxy is the expected Jacobian of the remaining layer stack:
\begin{equation}
P_\ell^* = \E_{z_\ell} \left[ \frac{\partial z_L}{\partial z_\ell} \right] 
\end{equation}
\end{proposition}

\begin{proof}
    Let's first resolve the closed-form solution for the parameters of the optimal linear proxy: $(P_\ell^*, b_\ell^*) = \operatorname*{arg\,min}_{P, b}\;\E\big\|\,z_L - Pz_\ell - b\big\|^2$. Setting the derivative with respect to $b_\ell$ to zero gives $b_\ell^\star = \E[z_L] - P\E[z_\ell]$, and substituting back, only the centered features remain, leaving us to minimize $\E\|\bar z_L - P \bar z_\ell\|^2$ where $\bar z_\ell = z_\ell - \E[z_\ell]$. Setting the derivative with respect to $P$ to zero gives $P_\ell^* = \Cov(\bar z_L, \bar z_\ell)\,\Sigma_{\ell}^{-1}$, where the positive definiteness of $\Sigma_\ell$ guarantees the existence of the inverse.
    %
    %
    Then we define $F_\ell(z_\ell) = f_L \circ f_{L-1} \dots \circ f_{\ell+1}(z_{\ell})$ which is also locally Lipschitz with $\E\|\partial z_L/\partial z_\ell\| < \infty$, and since $z_\ell$ is Gaussian with covariance $\Sigma_\ell$, the multivariate Stein's Lemma gives $\E[z_L \bar z_\ell^\top] = \E[\frac{\partial z_L}{\partial z_\ell}]\Sigma_\ell$. On the other hand, $\Cov(\bar z_L, \bar z_\ell) = \E[(z_L - \E[z_L]) \bar z_\ell^\top] = \E[z_L \bar z_\ell^\top] - \E[z_L] \E[\bar z_\ell]^\top = \E[z_L \bar z_\ell^\top]$ because $\E[\bar z_\ell]=0$. Therefore, we can replace the outer product of Stein's Lemma to get: $\E[\frac{\partial z_L}{\partial z_\ell}] = \E[z_L \bar z_\ell^\top] \Sigma_\ell^{-1} = \Cov(\bar z_L, \bar z_\ell)\Sigma_\ell^{-1} = P_\ell^*$.
    %
    %
    %
\end{proof}
%
%
\begin{corollary}
\label{thm-exact-linear} If the input $z_0$ is Gaussian with positive definite covariance $\Sigma_0$, and $F_\ell$ is a deep linear network ($f_{\theta}^\ell$ linear for all $\ell$) with full-row-rank feedforward weights, ProxyProp (Equation \ref{eq:proxygrad}) with linear proxy $\phi^*$ implements exactly the same algorithm as backpropagation:
\begin{equation}
    \text{(ProxyProp)} ~ g_{\theta}^{loc} = g_{\theta}^{BP} ~  \text{(backpropagation)}
\end{equation}
\end{corollary}
\begin{proof}
    Since the network is linear, each layer activation $z_\ell$ is simply a linear transformation of the input and is therefore also Gaussian. The covariance at layer $\ell$ follows $\Sigma_\ell = W_\ell \Sigma_{\ell-1} W_\ell^T$ where $W_\ell$ is the linear feedforward weight. 
    With full-row-rank $W_\ell$, $\Sigma_{\ell-1}$ being positive definite implies positive definiteness of $\Sigma_\ell$. As $\Sigma_0$ is positive definite, $\Sigma_\ell$ is positive definite for all $\ell$.  
    Moreover, each $F_\ell$ is linear and therefore also locally Lipschitz. 
    By Proposition\ref{thm-exact}, $P_\ell^* = \E[\frac{\partial z_L}{\partial z_\ell}]$. Because $F_\ell$ is a linear function, $\frac{\partial z_L}{\partial z_\ell}$ has no dependence on $z_\ell$, so $\frac{\partial z_L}{\partial z_\ell} = \mathbb{E}[\frac{\partial z_L}{\partial z_\ell}] = P_\ell^* $. Substituting this into Equation \ref{eq:proxygrad}, we find that the ProxyProp gradient is exactly the backpropagated gradient.
\end{proof}
\vspace{-0.3cm}
\textbf{Numerical verification in linear networks.}
We verify Proposition \ref{thm-exact} and Corollary \ref{thm-exact-linear} in a toy setting with linear networks.
We consider a six-layer MLP fed with batched images from the MNIST dataset \citep{lecun2010mnist}. When the MLP activation function is linear (squares in Figure \ref{subfig:mlpgradalign} left), the gradient $g_\theta^{loc}$ and the true gradient $g_\theta^{BP}$ align perfectly as predicted by Corollary \ref{thm-exact-linear}.

We then study the accuracy of a linear proxy in non-linear networks. 
In this case, even perfect estimation of the expected Jacobian $\E[{\partial z_L}/{\partial z_\ell}]$ does not mean good estimation of the gradient.
With non-linearities, the expected Jacobian does not equal the sample-specific Jacobian $\partial z_L/\partial z_\ell$ \citep{delrocq2026self}. 
Indeed, the gradient alignment is imperfect in our toy setting (circles in Figure \ref{subfig:mlpgradalign}), but ProxyProp approaches the best linear projection achievable (computed using the true gradient: $P' = \operatorname{argmin}_P || {\partial \mathcal{L}}/{\partial z_\ell} - P^\top {\partial \mathcal{L}}/{\partial z_L}||^2$).
As a comparison, DFA \citep{nokland2016direct} (which uses a random projection) has little alignment (crosses).
In our toy MLP setting, we demonstrate that the proxy distillation loss of Equation \ref{eq:proxyfunc} builds alignment without explicitly minimizing the difference of the gradients.


\textbf{Gradient consistency in non-linear proxies. }
As shown in Figure \ref{subfig:mlpgradalign}, even the very best linear proxy cannot yield perfect alignment in non-linear network.
The problem is further exacerbated in transformers taking sequential inputs.
Token-wise linear proxies are time invariant, so linear projections cannot track the temporal computation performed by attention layers in transformers.
In such cases, the proxy ideally is a non-linear sequence model, so we choose $p_\phi^\ell$ to be a transformer itself but with smaller depth and width. We use trainable linear projectors to map the hidden dimension of the network to that of the proxies. 

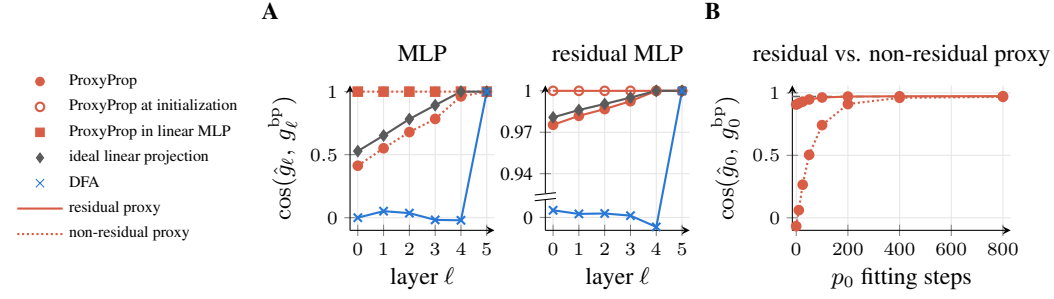
\begin{figure}
    \centering
    \providecommand{\figwidth}{}%
    \providecommand{\figheight}{5cm}%
    \sbox{\legendbox}{%
    \begin{tikzpicture}
    \begin{axis}[
      hide axis,
      scale only axis, width=1pt, height=1pt,
      xmin=0, xmax=1, ymin=0, ymax=1,
      legend columns=1,
      legend cell align=left,
      legend style={at={(0.5,0.5)}, anchor=center,
        font=\tiny, draw=none, fill=none, inner sep=0pt},
    ]
    \addlegendimage{only marks, cLinear, mark=*, mark size=1.6pt}
    \addlegendentry{ProxyProp}
    \addlegendimage{only marks, cLinear, mark=o, mark size=1.6pt, thick}
    \addlegendentry{ProxyProp at initialization}
    \addlegendimage{only marks, cLinear, mark=square*, mark size=1.6pt}
    \addlegendentry{ProxyProp in linear MLP}
    \addlegendimage{only marks, ide, mark=diamond*, mark size=1.8pt}
    \addlegendentry{ideal linear projection}
    \addlegendimage{only marks, cDFA, mark=x, mark size=2.2pt}
    \addlegendentry{DFA}
    \addlegendimage{cLinear, thick}
    \addlegendentry{residual proxy}
    \addlegendimage{cLinear, thick, densely dotted}
    \addlegendentry{non-residual proxy}
    \end{axis}
    \end{tikzpicture}}%
    \setlength{\panelAwidth}{\dimexpr(\linewidth-\wd\legendbox-1em)*64/100\relax}%
    \setlength{\panelBwidth}{\dimexpr(\linewidth-\wd\legendbox-1em)*32/100\relax}%
    \hspace*{\dimexpr\wd\legendbox+1em\relax}%
    \subfloat[]{%
        \renewcommand{\figwidth}{\panelAwidth}%
        \scriptsize%
        \label{subfig:mlpgradalign}%
        \let\legenddrawn\undefined
        \tikzset{every picture/.append style={execute at end picture={%
            \ifx\legenddrawn\undefined
              \node[overlay, anchor=east, inner sep=0pt, xshift=-1em]
                at (current bounding box.west |- current axis.center)
                {\usebox{\legendbox}};%
              \global\let\legenddrawn\relax
            \fi}}}%
        \begin{tikzpicture}
\begin{axis}[
  name=left,
  width=0.25\linewidth, height=3.5cm,
  xlabel={layer $\ell$},
  ylabel={$\cos(\hat g_\ell,\, g^{\mathrm{bp}}_\ell)$},
  title={\small MLP},
  xmin=-0.3, xmax=5.3, ymin=-0.1, ymax=1.05,
  xtick={0,1,2,3,4,5}, axis lines=left,
  grid=major, grid style={black!10},
  tick label style={font=\scriptsize}, label style={font=\small},
  clip=false]
\addplot[cLinear, thick, densely dotted, mark=square*, mark options={solid}, mark size=1.6pt] coordinates {(0,1.0000) (1,1.0000) (2,1.0000) (3,1.0000) (4,1.0000) (5,1.0000)};
\addplot[cLinear, thick, densely dotted, mark=*, mark size=1.6pt,
  mark options={solid}] coordinates {(0,0.4124) (1,0.5509) (2,0.6792) (3,0.7833) (4,0.9622) (5,1.0000)};
\addplot[ide, thick, mark=diamond*, mark size=1.8pt,
  mark options={solid}] coordinates {(0,0.5281) (1,0.6517) (2,0.7824) (3,0.8923) (4,1.0000) (5,1.0000)};
\addplot[cDFA, thick, mark=x, mark size=2.2pt,
  mark options={solid}] coordinates {(0,0.0003) (1,0.0520) (2,0.0366) (3,-0.0172) (4,-0.0203) (5,1.0000)};
\end{axis}
\hspace{-15pt}

\begin{axis}[
  at={(left.south east)}, anchor=south west, xshift=1.2cm,
  width=0.25\linewidth, height=3.5cm,
  xlabel={layer $\ell$},
  yticklabels={},
  title={\small residual MLP},
  xmin=-0.3, xmax=5.3, 
  y coord trafo/.code={\pgfmathparse{#1 > 0.5 ? (#1-0.93)*12+0.26 : #1}},
  y coord inv trafo/.code={\pgfmathparse{#1 > 0.2 ? (#1-0.26)/12+0.93 : #1}},
  ymin=-0.11225, ymax=1.004,
  ytick={0,0.94,0.97,1},
  yticklabels={0,0.94,0.97,1},
  xtick={0,1,2,3,4,5}, axis lines=left,
  grid=major, grid style={black!10},
  tick label style={font=\scriptsize}, label style={font=\small},
  legend style={font=\tiny, draw=none, fill=none,
    at={(-2.5,1)}, anchor=north east},
  legend cell align=right,legend plot pos=left,
  clip=false]
\addlegendimage{cLinear, thick, mark=square*, mark size=1.6pt}
\addplot[cLinear, thick, mark=o, mark size=1.6pt,
  mark options={solid}] coordinates {(0,1.0000) (1,1.0000) (2,1.0000) (3,1.0000) (4,1.0000) (5,1.0000)};
\addplot[cLinear, thick, mark=*, mark size=1.6pt,
  mark options={solid}] coordinates {(0,0.9754) (1,0.9819) (2,0.9868) (3,0.9925) (4,1.0000) (5,1.0000)};
\addplot[ide, thick, mark=diamond*, mark size=1.8pt,
  mark options={solid}] coordinates {(0,0.9808) (1,0.9861) (2,0.9905) (3,0.9951) (4,1.0000) (5,1.0000)};
\addplot[cDFA, thick,, mark=x, mark size=2.2pt,
  mark options={solid}] coordinates {(0,0.0619) (1,0.0305) (2,0.0343) (3,0.0155) (4,-0.0827) (5,1.0000)};
\fill[white] ([xshift=-4pt]axis cs:-0.3,0.15) rectangle ([xshift=4pt]axis cs:-0.3,0.21);
\draw ([xshift=-4pt]axis cs:-0.3,0.145) -- ([xshift=4pt]axis cs:-0.3,0.165);
\draw ([xshift=-4pt]axis cs:-0.3,0.195) -- ([xshift=4pt]axis cs:-0.3,0.215);
\end{axis}
\end{tikzpicture}%
    }
    \hspace*{-15pt}
    \subfloat[]{%
        \renewcommand{\figwidth}{\panelBwidth}%
        \scriptsize%
        \label{subfig:residpred}%
        \begin{tikzpicture}
\begin{axis}[
  width=0.32\linewidth, height=3.5cm,
  xlabel={$p_0$ fitting steps},
  title={\small residual vs. non-residual proxy},
  ylabel={$\cos(\hat g_0,\, g^{\mathrm{bp}}_0)$},
  xmin=-15, xmax=830, ymin=-0.1, ymax=1.06,
  axis lines=left, grid=major, grid style={black!10},
  tick label style={font=\scriptsize}, label style={font=\small},
  clip=false]
\addplot[black!60, dashed, semithick]
  coordinates {(-16,0.9703) (832,0.9703)};
\addplot[cLinear, thick, mark=*, mark size=1.6pt] coordinates {(0,0.9061) (10,0.9162) (25,0.9294) (50,0.9469) (100,0.9626) (200,0.9701) (400,0.9723) (800,0.9733)} node[pos=0, anchor=north west, font=\scriptsize, color=cLinear, yshift=-4pt]
  {};
\addplot[cLinear, thick, densely dotted, mark=*, mark size=1.6pt,
  mark options={solid}] coordinates {(0,-0.0669) (10,0.0617) (25,0.2652) (50,0.5039) (100,0.7415) (200,0.9101) (400,0.9601) (800,0.9675)} node[pos=0, anchor=north west, font=\scriptsize, color=cLinear, yshift=5pt, xshift=5pt]
  {};
\end{axis}
\end{tikzpicture}%
    }%
    \caption{\textbf{Numerical illustration of the theory.}
    \textbf{\ref{subfig:mlpgradalign}}: (Left) Alignment of the ProxyProp and backpropagation gradients in a random 6-layer MLP processing MNIST digits. (Right) Same as left, but the alignment is higher because of the skip connections. 
    \textbf{\ref{subfig:residpred}}: In the same residual setting, gradient alignment improves when the proxy objective is optimized (the network remains frozen). The residual proxy (solid) aligns already at the first iteration and converges faster.
    }
    \label{fig:relu_mlp_exactness}
\end{figure}

\textbf{Gradient consistency in residual proxies.} Modern architectures (e.g. transformers) rely on residual connections~\citep{he2016resnet}.
We design ProxyProp to leverage this to improve the gradient estimations. 
Below we show that the gradient approximation of ProxyProp is exact in residual networks at initialization (Remark \ref{thm-naive-residual}), and the error is bounded whenever the residual contribution at each layer is small in comparison with the corresponding skip connection (Proposition \ref{thm-early-training}). 

Let's first consider the initialization. The preferred initialization in residual networks and transformers enforces $f_{\theta}^\ell ( z_\ell )=z_\ell$ by initializing to zero the output weights~\citep{zhang2018zeroresweight} or by using a scalar to gate the output activation~\citep{goyal2017resnetinit, bachlechner2021rezero}. At least during the early phase of training, the skip connections dominate the computation $z_L \approx z_\ell$, which reflects during gradient propagation as ${\partial \mathcal{L}}/{\partial z_\ell} \approx {\partial \mathcal{L}}/{\partial z_L}$.
In ProxyProp, a trivial identity proxy is therefore sufficient to propagate an accurate approximation during the first training iterations. More generally, we can choose any (non-linear) residual proxy of the form $p_{\phi}^\ell(z_\ell) = z_\ell + q_{\phi}^\ell(z_\ell)$ with $q_\phi^\ell=0$ at initialization. 
We start with a simple remark regarding the benefit of this strategy (the proof is a direct consequence of the large network and the proxy being initialized to identity by construction):
\begin{remark}
    \label{thm-naive-residual}
    Consider any (non-linear) residual network and a residual proxy using standard zeroed residual initialization. Then, ProxyProp computes the exact error gradient at the first iteration.
\end{remark}
More broadly in any residual architecture, we can bound the gradient approximation error of ProxyProp whenever the residual contributions are small in comparison with the skip connections.
To formalize this, we assume that each layer is residual $f_\theta^\ell(z_\ell) = z_\ell + r_\theta^\ell (z_\ell)$, such that the deviations of the residual $||r_\theta^\ell (z) - r_\theta^\ell (z')||$ are always smaller than the deviations in the skip connections $|| z - z'||$.
When $z$ and $z'$ are close, this is equivalent to assuming that the operator norm of the Jacobian of the residual is bounded $|| {\partial} r_\theta/ \partial z ||_{op} \leq \epsilon$, whereas the norm of the Jacobian through the skip connection is one.
This is expected when the weights are still close to the zeroed initialization of Remark \ref{thm-naive-residual}.
Then, if we parameterize the proxy $p_\phi^\ell(z_\ell) = z_\ell + q_\phi^\ell(z_\ell)$ such that $|| {\partial} q_\phi/ \partial z ||_{op} \leq \epsilon$, we have (detailed proof in the Appendix \ref{appendix:proof}):
\begin{wrapfigure}{r}{0.45\linewidth}
  \centering
  \pgfplotsset{compat=1.17}
\begin{tikzpicture}
\begin{axis}[
  width=6cm, height=6cm,
  xlabel={$p_1$ fitting steps},
  ylabel={$\cos(\hat g_1,\, g^{\mathrm{bp}}_1)$},
  xmin=0, xmax=1950, ymin=0, ymax=0.7,
  axis lines=left, grid=major, grid style={black!10},
  tick label style={font=\scriptsize}, label style={font=\small},
    legend style={font=\scriptsize, draw=none, fill=none, at={(0.97,0.05)}, anchor=south east},
  legend cell align=right, legend plot pos=right, clip=false]]
\addplot[cReuse, thick] coordinates {(0,0.4850) (50,0.4938) (100,0.4999) (150,0.4875) (200,0.4912) (250,0.4962) (300,0.5009) (350,0.5022) (400,0.4977) (450,0.5062) (500,0.5066) (550,0.5103) (600,0.5134) (650,0.5188) (700,0.5211) (750,0.5257) (800,0.5250) (850,0.5315) (900,0.5298) (950,0.5335) (1000,0.5411) (1050,0.5438) (1100,0.5454) (1150,0.5471) (1200,0.5577) (1250,0.5489) (1300,0.5443) (1350,0.5551) (1400,0.5504) (1450,0.5570) (1500,0.5583) (1550,0.5595) (1600,0.5583) (1650,0.5671) (1700,0.5613) (1750,0.5662) (1800,0.5711) (1850,0.5630) (1900,0.5692)};
\addlegendentry{transformer proxy, residual}
\addplot[cBP, thick] coordinates {(0,0.3398) (50,0.4034) (100,0.4403) (150,0.4609) (200,0.4727) (250,0.4795) (300,0.4815) (350,0.4864) (400,0.4890) (450,0.4908) (500,0.4869) (550,0.4913) (600,0.4934) (650,0.4917) (700,0.4936) (750,0.4920) (800,0.4938) (850,0.4956) (900,0.4952) (950,0.4932) (1000,0.4950) (1050,0.4975) (1100,0.4935) (1150,0.4945) (1200,0.4947) (1250,0.4973) (1300,0.4974) (1350,0.4947) (1400,0.4971) (1450,0.4954) (1500,0.4969) (1550,0.4960) (1600,0.4930) (1650,0.4958) (1700,0.4963) (1750,0.4964) (1800,0.4949) (1850,0.4924) (1900,0.4961)};
\addlegendentry{linear proxy, residual}
\addplot[cReuse, thick, dashed] coordinates {(0,0.0610) (50,0.1925) (100,0.2632) (150,0.2825) (200,0.3088) (250,0.3269) (300,0.3475) (350,0.3637) (400,0.3774) (450,0.3908) (500,0.4012) (550,0.4079) (600,0.4227) (650,0.4326) (700,0.4445) (750,0.4471) (800,0.4444) (850,0.4492) (900,0.4480) (950,0.4553) (1000,0.4538) (1050,0.4586) (1100,0.4611) (1150,0.4641) (1200,0.4695) (1250,0.4689) (1300,0.4723) (1350,0.4773) (1400,0.4819) (1450,0.4824) (1500,0.4947) (1550,0.4909) (1600,0.4890) (1650,0.4968) (1700,0.4917) (1750,0.4964) (1800,0.4981) (1850,0.4967) (1900,0.4996)};
\addlegendentry{transformer proxy, non-residual}
\addplot[cBP, thick, dashed] coordinates {(0,0.0166) (50,0.1348) (100,0.2528) (150,0.3340) (200,0.3839) (250,0.4142) (300,0.4287) (350,0.4435) (400,0.4519) (450,0.4555) (500,0.4537) (550,0.4604) (600,0.4640) (650,0.4632) (700,0.4658) (750,0.4647) (800,0.4672) (850,0.4692) (900,0.4691) (950,0.4673) (1000,0.4697) (1050,0.4723) (1100,0.4682) (1150,0.4705) (1200,0.4703) (1250,0.4733) (1300,0.4732) (1350,0.4706) (1400,0.4736) (1450,0.4718) (1500,0.4733) (1550,0.4728) (1600,0.4698) (1650,0.4735) (1700,0.4738) (1750,0.4742) (1800,0.4727) (1850,0.4697) (1900,0.4744)};
\addlegendentry{linear proxy, non-residual}
\end{axis}
\end{tikzpicture}
  \caption{\textbf{Non-linear proxy with residual connection improves gradient alignment in trained transformers.} Cosine similarity between the ProxyProp and the backpropagation gradients for the second block of a 12-layer ViT (partitioned into three blocks of four layers).
  \vspace{-0.3cm}
  }
  \label{fig:proxy-target}
\end{wrapfigure}
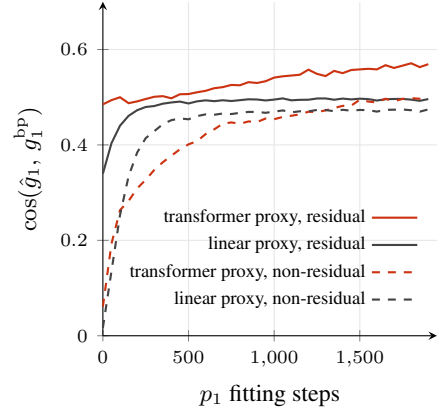
\begin{proposition}
    \label{thm-early-training}
    Consider a deep (non-linear) residual network which is dominated by the skip connections (at all layers we assume $|| {\partial} r_\theta/ \partial z ||_{op} \leq \epsilon$). Then, given a fixed depth $L$, if we choose a (non-linear) residual proxy with $|| {\partial} q_\phi/ \partial z ||_{op} \leq \epsilon$, there exists a constant $C$ (given explicitly in the Appendix \ref{appendix:proof}) such that the gradient estimation error is bounded as follows:
    $$ || g_{\theta_\ell}^{BP} - g_{\theta_\ell}^{loc} || \leq C \epsilon + O(\epsilon^2) $$
\end{proposition}
\textbf{Numerical advantage of residual proxy.} To verify the impact of the residual proxy during pretraining, we construct a 6-layer MLP with skip connections at each layer: $z_{\ell+1} = z_\ell + W_{\ell, 2} \text{ReLU}(W_{\ell, 1}z_\ell)$.
We confirm Remark \ref{thm-naive-residual} by showing that with zeroed initialization ($W_{\ell, 2}=0$), the gradients are perfectly aligned at all layers in Figure \ref{subfig:mlpgradalign} (right).
Away from initialization (non-zero $W_{\ell,2}$), $|| {\partial} r_\theta/ \partial z ||_{op}$ is not necessarily small, but gradient alignment of the residual proxy is above $0.9$, approaching the best possible alignment. Even when the distillation loss is not yet fully minimized, the residual proxy gradient aligns remarkably well (Figure \ref{subfig:residpred}). Distillation eventually improves the alignment further.

With trained transformers, approaching perfect alignment is much harder, though the residual parameterization still shows a clear advantage. 
To show this empirically, we trained a ViT on CIFAR-100 classification \citep{Krizhevsky09learningmultiple} for 100 epochs with end-to-end backpropagation.
Then, we partition the ViT into three blocks and train a proxy on the penultimate block with a frozen backbone.
We observe in Figure \ref{fig:proxy-target} that, firstly, transformer proxies are better than linear proxies, and secondly, a residual proxy parameterization improves the alignment.






\section{Accuracy/memory trade-off in transformers}
\label{sec:results}

\definecolor{cReuse}{HTML}{CC3311}  
\definecolor{cLinear}{HTML}{D85D44}  
\definecolor{cIdentity}{HTML}{E37E68}  
\definecolor{cRandom}{HTML}{ED9D8B}  
\definecolor{cNores}{HTML}{FF7F0E}  
\definecolor{cShared}{HTML}{8B4486}  
\definecolor{cDFA}{HTML}{2A78D6}  
\definecolor{cDRTP}{HTML}{009988}  
\definecolor{cDNI}{HTML}{332288}  
\definecolor{cDBlock}{HTML}{117733}  
\definecolor{cBP}{HTML}{444444}  
\definecolor{cMeZO}{HTML}{20B2AA} 


\textbf{Studies in vision transformers.}
To test empirically if the theoretical consistency of the proxy gradient translates to training transformers, we first implement a 12-layer ViT \citep{dosovitskiy2021an} with embedding dimension $d=128$ on CIFAR-100. 
The network is partitioned into $L=3$, $6$, or $12$ blocks, the three cases being denoted $3\times4$, $6\times2$, and $12\times1$ respectively to represent the number of layers per block.
Gradients are computed either with ProxyProp or alternative local learning algorithms. ProxyProp parameterizes a proxy for each block as a small transformer of lower dimension $d=64$. 
ProxyProp ($3\times4$) achieves performance close to backpropagation while reducing the VRAM requirements by 62\% (Figure \ref{subfig1b:accmem}).

\begin{table}
    \centering
    \small
    \caption{\textbf{ProxyProp offers a better trade off between accuracy and memory cost}. 
    We report here the results on CIFAR-100 with a 12-layer ViT. We use a relatively large batch of 512 to highlight the memory gain of the activation vector. Our backpropagation baseline with the same architecture and training hyperparameters achieved $66.13\%$ accuracy with 12.30 GB of peak VRAM usage. 
    The accuracy number is in bold when it is highest or non-significantly lower (one-sided paired $t$-test over 5 seeds, Holm-corrected; $\,p<0.05$).}
    \label{tab:b3-baselines}
    \begin{tabular}{lcccc}
    \hline
     Method & & top-1 accuracy & & VRAM (GB)\\
     & {$\scriptstyle 3\times4$} & {$\scriptstyle 6\times2$} & {$\scriptstyle 12\times1$} & {$\scriptstyle 3\times4/6\times2/12\times1$} \\
    \hline
    ProxyProp & $\boldsymbol{64.90} {\scriptstyle \pm 0.44}$ & $\boldsymbol{61.51} {\scriptstyle \pm 0.19}$ & $\boldsymbol{57.58} {\scriptstyle {\scriptstyle \pm 0.40}}$ & 4.64 / 2.76 / 1.90 \\
    \quad with DNI style loss & $\boldsymbol{64.43} {\scriptstyle \pm 0.54}$ & $60.42 {\scriptstyle \pm 0.63}$ & $\boldsymbol{58.31} {\scriptstyle \pm 0.60}$ & 13.28 / 7.92 / 5.10 \\
    \quad residual linear proxy & $63.18 {\scriptstyle \pm 0.56}$ & $58.86 {\scriptstyle \pm 0.32}$ & $47.67 {\scriptstyle \pm 0.27}$ & 4.62 / 2.69 / 1.72 \\
    \quad identity proxy & $58.15 {\scriptstyle \pm 0.84}$ & $57.08 {\scriptstyle \pm 0.57}$ & $46.07 {\scriptstyle \pm 0.41}$ & 4.62 / 2.69 / 1.72 \\
    \hline
    Backprop (one block) & $63.93 {\scriptstyle \pm 0.60}$ & $60.00 {\scriptstyle \pm 0.54}$ & $50.33 {\scriptstyle \pm 0.32}$ & 4.44 / 2.47 / 1.49 \\
    Local loss \citep{belilovsky2019greedy} & $64.06 {\scriptstyle \pm 0.26}$ & $60.15 {\scriptstyle \pm 0.45}$ & $51.96 {\scriptstyle \pm 0.44}$ & 4.57 / 2.63 / 1.66 \\
    DFA \citep{nokland2016direct}       & $55
     .32 {\scriptstyle \pm 0.19}$ & $49.96 {\scriptstyle \pm 0.37}$ & $35.66 {\scriptstyle \pm 0.24}$ & 4.62 / 2.74 / 1.87 \\
    DRTP \citep{frenkel2021learning}    & $58.56 {\scriptstyle \pm 0.41}$ & $54.66 {\scriptstyle \pm 0.28}$ & $34.97 {\scriptstyle \pm 0.66}$ & 4.62 / 2.74 / 1.87 \\
    DNI \citep{jaderberg2017syngrad}   & $41.39 {\scriptstyle \pm 0.45}$ & $33.11 {\scriptstyle \pm 0.34}$ & $16.67 {\scriptstyle \pm 3.58}$ & 4.66 / 2.72 / 1.76 \\
    \end{tabular}
    \vspace{-0.8\baselineskip}
  \end{table}
To justify the design choices of the proxies, we implement ProxyProp variants and compare their classification performance (\autoref{tab:b3-baselines}).
The simplest variant is when the proxy is reduced to identity (the error gradient ${\partial \mathcal{L}}/{\partial z_L}$ is applied at every layer without transformations). 
Despite its simplicity, it outperforms local learning algorithms that use random fixed matrices to train the blocks (see DFA and DRTP in Table \ref{tab:b3-baselines}). 
This observation supports the advantage of leveraging the residual structure when estimating local gradients (Remark \ref{thm-naive-residual} and Proposition \ref{thm-early-training}).
Consistent with our measurements on gradient alignment, non-linear transformer proxies perform better than linear ones.

Then, we examine if the simple proxy distillation objective could be replaced with loss functions that more explicitly target the gradient. 
We consider the DNI gradient bootstrapping loss from ~\cite{jaderberg2017syngrad}. 
Following their method, we train an auxiliary MLP to estimate a gradient bootstrapped recursively from the auxiliary network of the next layer: with our ViT architecture this method is not stable. 
Instead, we use ProxyProp but replace the distillation loss (Equation \ref{eq:proxyfunc}) with the DNI loss that explicitly estimates the global gradient. 
Despite the increase in memory from the second-order derivatives, it does not bring any significant improvement over our simpler distillation loss (see Table \ref{tab:b3-baselines}). 
This empirical finding supports, along with Proposition \ref{thm-exact} and Corollary\ref{thm-exact-linear}, that proxy distillation is a simple and efficient way to align the local gradients with the global gradients.

We measured for all methods the VRAM occupied in the GPU for the various training algorithms.
Given a relatively large batch size, the memory consumption is dominated by the storage of the network activations to illustrate the benefits of the algorithm.
By choosing the number of layers for each block in the network, we can navigate the trade-off between network performance and memory consumption (either 3, 6, or 12 blocks).
While ProxyProp pays a small storage to propagate gradients through the proxy, it sits on the Pareto front of the accuracy/memory trade-off (Figure~\ref{subfig1b:accmem}).

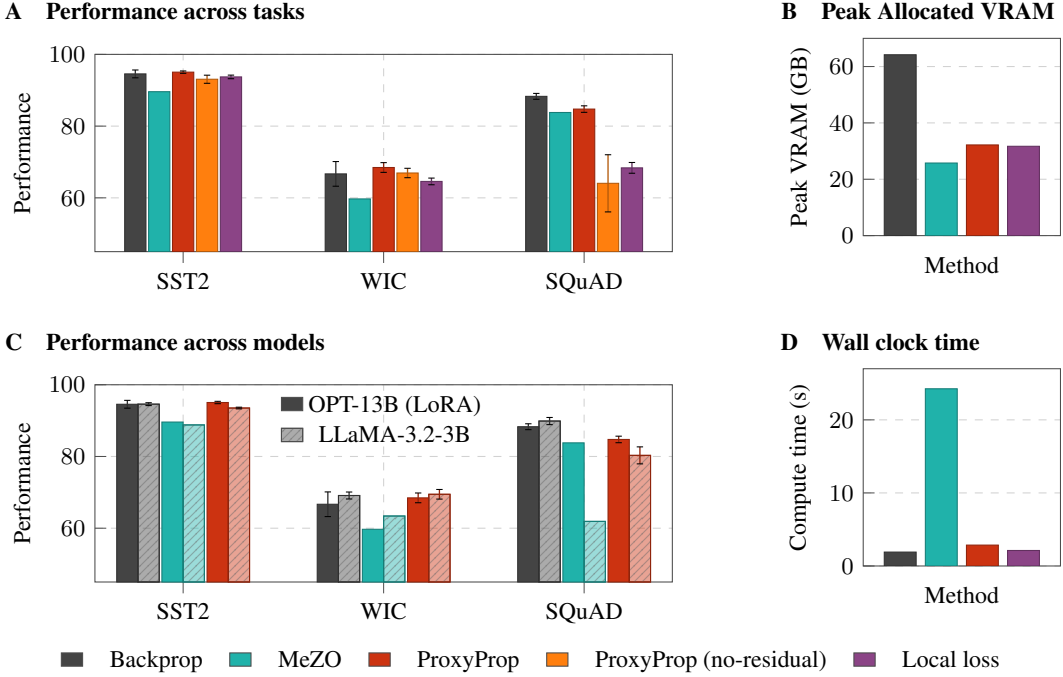
\begin{figure}[thb]
\vspace{-\baselineskip}
  \pgfplotsset{
    perf axis/.style={
      ybar=1pt, bar width=8pt, enlarge x limits=0.22, ymin=0,
      /pgfplots/error bars/y dir=both,
      /pgfplots/error bars/y explicit,
      /pgfplots/error bars/error bar style={
        line width=0.5pt
      },
      /pgfplots/error bars/error mark options={
        rotate=90,
        mark size=1.2pt
      },
      grid=major, grid style={dashed,gray!35},
      axis line style={black!70}, tick pos=left, tick style={black!70},
      tick label style={font=\small}, label style={font=\small},
      legend style={font=\footnotesize, draw=none, fill=none},
    },
    task axis/.style={perf axis, width=0.62\linewidth, height=0.3\linewidth,
      ylabel={Performance}, symbolic x coords={SST2,WIC,SQuAD}, xtick=data, ymin=30, ymax=100},
    cost axis/.style={perf axis, width=0.3\linewidth, height=0.3\linewidth, bar width=12pt,
      symbolic x coords={BP,MeZO,PP,SH}, xtick=data, xtick=\empty, xlabel={Method},
      enlarge x limits=0.3, xmajorgrids=false, every axis plot/.append style={bar shift=0pt}},
    method/.style={fill=#1, draw=#1!70!black,
      legend image code/.code={\draw[fill=#1, draw=black!70] (0pt,-2pt) rectangle (8pt,3pt);}},
    llama/.style={fill opacity=0.45, postaction={pattern=north east lines, pattern color=black!60}},
  }
    \centering
    \subfloat[\textbf{Performance across tasks}]{%
        \begin{tikzpicture}
        \begin{axis}[task axis, legend to name=perfLegend, legend columns=5,
                     legend style={column sep=0.8em}, ymin=45,
                      width=0.66\linewidth,
                      height=0.3\linewidth,]
        \addplot[method=cBP]     coordinates {(SST2,94.56) +- (0,1.09) (WIC,66.68) +- (0,3.44) (SQuAD,88.30) +- (0,0.82)}; \addlegendentry{Backprop}
        \addplot[method=cMeZO]   coordinates {(SST2,89.60)  (WIC,59.70)  (SQuAD,83.80) }; \addlegendentry{MeZO}
        \addplot[method=cReuse]  coordinates {(SST2,95.05) +- (0,0.33) (WIC,68.46) +- (0,1.36) (SQuAD,84.75) +- (0,0.91)}; \addlegendentry{ProxyProp}
        \addplot[method=cNores]  coordinates {(SST2,93.07) +- (0,1.13) (WIC,66.93) +- (0,1.31) (SQuAD,64.05) +- (0,7.97)}; \addlegendentry{ProxyProp (no-residual)}
        \addplot[method=cShared] coordinates {(SST2,93.72) +- (0,0.50) (WIC,64.58) +- (0,0.93) (SQuAD,68.36) +- (0,1.50)}; \addlegendentry{Local loss}
        \end{axis}
        \label{subfig:perf_tasks}
        \end{tikzpicture}}%
    \hfill
    \subfloat[\textbf{Peak Allocated VRAM}]{%
        \begin{tikzpicture}
        \begin{axis}[cost axis, ylabel={Peak VRAM (GB)}, ymax=70,
        width=0.30\linewidth,
        height=0.3\linewidth,]
        \addplot[method=cBP]     coordinates {(BP,64.18)};
        \addplot[method=cMeZO]   coordinates {(MeZO,25.76)};
        \addplot[method=cReuse]  coordinates {(PP,32.18)};
        \addplot[method=cShared] coordinates {(SH,31.70)};
        \end{axis}
        \label{subfig:peakgpu}
        \end{tikzpicture}}
    \\
    \subfloat[\textbf{Performance across models}]{%
        \begin{tikzpicture}
        \begin{axis}[task axis, legend columns=1, legend style={at={(0.5,0.99)}, anchor=north, inner sep=1pt},
                      width=0.66\linewidth, ymin=45,
                      height=0.3\linewidth,]
        \addplot[method=cBP,    bar shift=-21pt]        coordinates {(SST2,94.56) +- (0,1.09) (WIC,66.68) +- (0,3.44) (SQuAD,88.30) +- (0,0.82)};
        \addplot[method=cBP,    bar shift=-13pt, llama] coordinates {(SST2,94.59) +- (0,0.40) (WIC,69.12) +- (0,0.97) (SQuAD,89.88) +- (0,1.00)};
        \addplot[method=cMeZO,  bar shift=-4pt]         coordinates {(SST2,89.60) (WIC,59.70) (SQuAD,83.80)};
        \addplot[method=cMeZO,  bar shift=4pt,   llama] coordinates {(SST2,88.80) (WIC,63.40) (SQuAD,61.90)};
        \addplot[method=cReuse, bar shift=13pt]         coordinates {(SST2,95.05) +- (0,0.33) (WIC,68.46) +- (0,1.36) (SQuAD,84.75) +- (0,0.91)};
        \addplot[method=cReuse, bar shift=21pt,  llama] coordinates {(SST2,93.51) +- (0,0.23) (WIC,69.46) +- (0,1.34) (SQuAD,80.31) +- (0,2.37)};
        \addlegendimage{method=black!70}         \addlegendentry{OPT-13B (LoRA)}
        \addlegendimage{method=black!70, llama}  \addlegendentry{LLaMA-3.2-3B}
        \end{axis}
        \label{subfig:perf_models}
        \end{tikzpicture}}%
    \hfill
    \subfloat[\textbf{Wall clock time}]{%
        \begin{tikzpicture}
        \begin{axis}[cost axis, ylabel={Compute time (s)}, ymax=27, width=0.30\linewidth,
        height=0.3\linewidth,]
        \addplot[method=cBP]     coordinates {(BP,1.902)};
        \addplot[method=cMeZO]   coordinates {(MeZO,24.256)};
        \addplot[method=cReuse]  coordinates {(PP,2.859)};
        \addplot[method=cShared] coordinates {(SH,2.122)};
        \end{axis}
        \label{subfig:compute_time}
        \end{tikzpicture}}
    \\[2pt]
    \pgfplotslegendfromname{perfLegend}   
    \\[6pt]
    \caption{\textbf{ProxyProp fine-tunes LLMs with better tradeoff between performance, memory, and compute time.}
    \textbf{\ref{subfig:perf_tasks}} OPT-13B fine-tuned with rank-8 LoRA by Backprop, MeZO, and local algorithms; SST-2 and WiC (classification tasks) report accuracy, and SQuAD (question answering task) reports F1. MeZO values are taken from \citet{malladi2023mezo,tan2025dizo}.  Error bars are the 95\% confidence interval computed from 5 different random seeds.
    \textbf{\ref{subfig:perf_models}} Performance on OPT-13B with LoRA (solid) and LLaMA-3.2-3B (hatched).
    \textbf{\ref{subfig:peakgpu}, \ref{subfig:compute_time}} Peak VRAM and compute time for one training step of input batch with size 8, sequence length 800 on A100 GPU. MeZO's compute time is scaled up by its 16x number of training steps used in \citet{malladi2023mezo}.}

    \label{fig:llm_fine-tune}
    \vspace{-1\baselineskip}
\end{figure}

\textbf{Fine-tuning in large language models.}
We identify that LLM fine-tuning is a relevant use case where ProxyProp can be decisive.
While LLMs are pretrained in datacenters with extensive resources, fine-tuning these models is practicable with consumer-grade GPUs.
Choice of model and training context length are often limited by the VRAM capacity.
For fine-tuning, the memory storage can be roughly divided into three main categories: the weights, the optimizer, and the model activations.
The memory footprint of the weights may be reduced via quantization \citep{dettmers2023qlora} and the optimizer state via low-rank adapters (LoRA \citep{hu2022lora}); we study how local learning reduces the memory footprint of the activations.

A popular method to almost entirely eliminate the storage of the activations is to estimate gradients with repeated forward passes. 
These methods are often referred to as zeroth-order optimization techniques (e.g. MeZO, \cite{malladi2023mezo, tan2025dizo}); they are cheap in memory, but their noisy gradient estimates can lead to slow convergence and increase the number of required optimization steps.
We show below that ProxyProp offers a practical alternative.

First we demonstrate that ProxyProp is able to fine-tune LLMs.
We choose popular datasets and models with existing reference baselines for MeZO and backpropagation, then we fine-tune language models OPT-13B \citep{zhang2022opt} and LLaMA-3.2-3B \citep{dubey2024llama} on three datasets (1000 examples each): SST2, WIC, and SQuAD \citep{socher2013sst2,wang2019superglue,rajpurkar2016squad}. 
We split the entire model into either 7 or 8 blocks and use proxies of dimension 64 or 80 depending on the model (details in Appendix \ref{appendix:detail}). 
ProxyProp approaches the performance of backpropagation across the 3 tasks and 2 architectures, and also achieves higher performance than the MeZO results reported by \citet{malladi2023mezo,tan2025dizo} (Figure \autoref{subfig:perf_tasks}, \autoref{subfig:perf_models}).
We also confirm that alternative local algorithms struggle to achieve a comparable performance on these fine-tuning benchmarks (we tested ProxyProp without the residual parameterization and a block-wise local-loss baseline).

Our experiments show that ProxyProp uses resources efficiently.
The backpropagation baseline is implemented with the standard HuggingFace Trainer. 
Towards a fairer comparison, we inherit the trainer, swapping out PyTorch backpropagation with ProxyProp's computation of the gradients.
For OPT-13B, we use LoRA to reduce the memory consumption of the optimizer. 
On the same hardware, ProxyProp is roughly as fast as backpropagation (approximately $\times1.5$ slower).
In comparison, zeroth-order optimization becomes slow because it needs repetitions of parameter perturbations to reduce the variance of the gradient estimates. 
Even 16 times as many steps (as reported in \cite{malladi2023mezo}), which means eight times the wall-clock time (Figure \autoref{subfig:compute_time}), are not sufficient for MeZO to reach the performance of ProxyProp.
In terms of memory, MeZO requires less VRAM because it only needs to store the network weights (approximately $25.7$~GB). 
ProxyProp needs $32$~GB and backpropagation $64$~GB (see Figure \ref{subfig:peakgpu}).
ProxyProp achieves a strong memory-efficient performance without compromising training time.

To demonstrate the potential of ProxyProp to save memory even in combination with contemporary techniques, we combine it with LoRA, chunked loss, and gradient checkpointing (the default Hugging Face implementation that caches only the output of each transformer layer). 
We compare the memory scaling of training LLaMA-3.1-8B by backpropagation and by ProxyProp with artificial data at varying sequence lengths. 
As shown in Figure \ref{fig:llm_length}, ProxyProp with 4 blocks could fit in sequences over 50\% longer than backpropagation. 
As the number of blocks increases, ProxyProp fits in longer sequence and approaches the memory scaling of inference.

\begin{wrapfigure}{r}{0.47\linewidth}
  \vspace{-1.4\baselineskip}  
  \centering
  \input{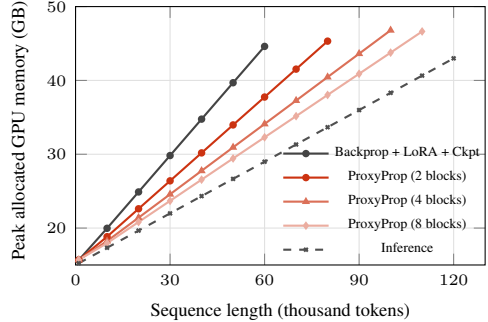}
  \definecolor{inf}{RGB}{80,80,80}
  \pgfplotsset{compat=1.18}
    \pgfplotstableread{
    length inference bp b2 b4 b8
    1.024 15.233918 15.667524 15.673038 15.681301 15.697841
    10 17.331387 19.970292 18.832973 18.236822 17.955286
    20 19.659499 24.894255 22.612527 21.406024 20.819313
    30 21.990998 29.822176 26.396039 24.579184 23.687296
    40 24.328051 34.755294 30.182795 27.755589 26.558526
    50 26.662664 39.685572 33.969884 30.931594 29.430820
    60 28.989021 44.606768 37.746915 34.099006 32.292079
    70 31.319559 49.533056 41.529038 37.270777 35.158431
    80 33.658047 54.467548 45.319365 40.450554 38.032498
    90 35.985359 59.390742 49.098150 43.619186 40.896244
    100 38.316896 64.317273 52.880272 46.790956 43.762838
    110 40.650051 69.246598 56.665187 49.965520 46.632226
    120 42.988134 74.180987 60.453214 53.143196 49.506237
    }\memorydata
  \begin{tikzpicture}
    \begin{axis}[
      width=7cm,height=5cm,
      xmin=0,xmax=130,xtick={0,30,60,90,120},
      xlabel={Sequence length (thousand tokens)},
      grid=major,grid style={black!12},
      tick label style={font=\tiny},label style={font=\scriptsize},
      title style={font=\scriptsize},
      every axis plot/.append style={thick,mark size=1pt},
      name=total,ymin=15,ymax=50,
      ylabel={Peak allocated GPU memory (GB)},
      ylabel style={xshift=2pt},
      legend pos=south east,
      legend style={at={(1.02,-0.01)},draw=none,font=\fontsize{5.2pt}{5.5pt}\selectfont,fill=none},
    ]

    \addplot[cBP,mark=*,restrict x to domain=1.024:60] table[x=length,y=bp] {\memorydata};
    \addlegendentry{Backprop + LoRA + Ckpt}
    \addplot[cReuse,mark=*,restrict x to domain=1.024:80] table[x=length,y=b2] {\memorydata};
    \addlegendentry{ProxyProp (2 blocks)}
    \addplot[cReuse!70!white,mark=triangle*,restrict x to domain=1.024:100] table[x=length,y=b4] {\memorydata};
    \addlegendentry{ProxyProp (4 blocks)}
    \addplot[cReuse!40!white,mark=diamond*,restrict x to domain=1.024:110] table[x=length,y=b8] {\memorydata};
    \addlegendentry{ProxyProp (8 blocks)}
    \addplot[inf,dashed,mark=x] table[x=length,y=inference] {\memorydata};
    \addlegendentry{Inference}

    \end{axis}
    \end{tikzpicture}
  \caption{
  \textbf{ProxyProp approaches inference memory usage.}
  We measure peak memory allocation when fine-tuning a LLaMA-8B model with varying-length sequence input.
  All baselines benefit from the standard memory-efficient fine-tuning implementation of the HuggingFace transformer library, including rank-8 LoRA, gradient checkpointing, and cross-entropy loss chunking. For ProxyProp we use a residual proxy transformer with 64 dimensions.
  }
  \label{fig:llm_length}
  \vspace{-0.5\baselineskip}
\end{wrapfigure}

\section{Discussion}
We analyzed theoretically how to compute local and memory-efficient gradient estimates in transformers.
We present ProxyProp as a local learning algorithm which estimates the global gradient with theoretical guarantees.
In practice, this algorithm trains ViTs and fine-tunes LLMs to near-backpropagation accuracy on a low memory budget.
Below we speculate the potential impact of local learning beyond what we have presented in previous sections.

\textbf{Local learning for very deep, looped, or diffusion architectures.}
The memory benefits of local learning are already visible on moderately deep LLM architectures (dozens of layers), but ProxyProp memory consumption is constant with the number of layers. From the viewpoint of memory, the benefits increase with deeper networks. 
Local learning has the potential to scale to much deeper architectures which are hard to train with standard backpropagation.

A case of interest is a looped transformer in which repeating the recurrence is effectively increasing the depth of the network. To avoid increasing the memory storage, the gradients are not propagated through all of the loops; rather, the gradient is propagated only through the last (or last few) loop iterations~\cite{yang2024looped}.
Local proxy propagation learning would be adapted to this setting, as it offers theoretical insights and a method to compute gradients which are aware of all the repeated loops without increasing memory storage. 
From another perspective, a single block in residual models is analogous to a diffusion step.
Pushing this analogy, \cite{li2026noprop} and \cite{shing2026diffusionblocks} train deep residual architectures using diffusion training schemes.
Reciprocally, our analyses might inspire methods to compute global gradients from a single local diffusion step.

\textbf{Decoupling parameter updates.}
The potential of local learning is not limited to memory saving. 
By solving the backward locking problem, local learning also enables model-parallelism during training \citep{jaderberg2017syngrad,belilovsky2020dgl}. 
While the computation speed on a single device is roughly similar between ProxyProp and backpropagation (see Table~\ref{tab:complexity}), decoupling becomes relevant when training across devices. 
As described, ProxyProp from Algorithm~\ref{alg:training} preserves a weak coupling. 
We see at least two directions to solve this issue: firstly, the pre-forward pass could be computed ahead of time for a large batch without coupling; with mini-batches then being dispatched in order and iterated. 
Hence, the algorithm can be applied if one accepts a lag between the loss gradients and the local block that it is then sent to. 
Alternatively, we describe in Algorithm \ref{alg:unlock} the fully unlocked ProxyProp variant that uses a pretrained proxy to train the network.
Although we have not studied distributed implementations of our algorithm at the moment, we feel that it is a promising direction for future work.

\subsection*{AI use statement}


We used generative AI tools to help review the literature exhaustively. We have used AI to brainstorm mathematical justifications for the algorithm we had planned to implement.
We used AI to generate code, but we implemented ourselves the critical core functions of the local learning algorithm which is central to this publication.
Less critical code, such as the implementation of a toy example or standard baselines, was quality controlled by the entire team using numerical and visual verifications.
We wrote the first versions of every mathematical proof ourselves, and used AI to improve the writing or mathematical consistency and verify our claims. All the final proofs have been reviewed by all the authors.

We have not used LLMs to seed or design the research direction.
We have not used AI to write text or ideas. All the text was written manually without AI, but AI was used to find typographic mistakes.
All the AI corrections were reviewed each time by one of the authors.
We have not used AI to generate autonomously the experiments or optimize the performance autonomously.
We did not use AI to recommend specific data analysis strategies, numerical experiments or visualization.
Provided with data from our experiments, we used AI to help format figures in tikz.
We take responsibility for the final content of this work,
including text, claims or artifacts produced with the aid of generative AI.

\subsection*{Reproducibility statement}
Upon acceptance of the paper, we will publish the PyTorch implementation and local learning wrapper for the Trainer class of the Hugging Face library on GitHub.
The datasets used in all experiments are freely and publicly available. 
All authors have been able to use our algorithm on different datasets and successfully trained a neural network model with consistent training speed and memory RAM requirement.
The results that are claimed were consistent across datasets, even when the experiments were implemented by different authors.

\subsubsection*{Acknowledgments}
Brian Andrew Yuan (BAY) and Wu S. Zihan (WSZ) contributed equally, and Guillaume Bellec (GB) acted as a research supervisor. All three authors contributed fully to seeding the research direction, the mathematical theory, running experiments, and writing the paper. BAY ran all the applications and comparisons of local learning on the vision transformer, WSZ ran the simulations for language models.

This research is supported by Vienna Science and Technology Fund (VRG24-018), the Swiss National Science Foundation grants (200021-236436), and Machinita Technologies. We thank Florian Chen, Ziao Wang, Qinyu Li, Ariane Delrocq, Wulfram Gerstner and everyone from the Machine Learning Research Unit of TU Wien for their feedback and useful discussions. 

\newpage
\bibliography{References}
\bibliographystyle{iclr2027_conference}

\newpage
\appendix
\renewcommand{\thetable}{A\arabic{table}}
\setcounter{table}{0}
\renewcommand{\thefigure}{A\arabic{figure}}
\setcounter{figure}{0}

\section{Scaling of memory and compute}

\begin{table}[h]
    \centering
    \small
    \caption{Scaling of compute and memory cost with training an $L$-block network. $^*$ For transformer, the optimal activation checkpointing is not $\mathcal{O}(\sqrt{L})$ because attention and MLP within each layer also requires plenty of memory storage. A common practice is to store the output of each transformer layer and recompute the activations of attention and MLP, which make the activation storage $\mathcal{O}(L)$.}
    \label{tab:complexity}
    \begin{tabular}{lccc}
    \hline
    Method & Forward FLOPs & Backward FLOPs & Activation storage \\
    \hline
    Backprop (BP) & $\mathcal{O}(L)$ & $\mathcal{O}(L)$ & $\mathcal{O}(L)$ \\
    BP + checkpointing & $\mathcal{O}(2 L)$ & $\mathcal{O}(L)$ & $ \mathcal{O}(\sqrt{L})^*$ \\
    ProxyProp & $\mathcal{O}(2L)$ & $\mathcal{O}(L)$ & $\mathcal{O}(1)$ \\
    \hline
    \end{tabular}
\end{table}

\begin{algorithm}[bth]
\caption{Description of the fully unlocked ProxyProp algorithm. The input $z_0$ is given. $p_{\phi}^L(z_l)$ is always the identity.}
\label{alg:unlock}
\vspace{4pt}

Pretrain the proxy parameters $\phi$ with the distillation loss from Equation~\ref{eq:proxyfunc}\
\vspace{4pt}

\For{$\ell = 1$ \KwTo $L$}{
    $z_\ell \gets f_\ell(z_{\ell-1})$\;
    Compute the proxy $p_{\phi}^{\ell}(z_\ell)$\;
    Compute loss $\mathcal{L}(\theta_h, p_{\phi}^\ell(z_\ell), y)$\;
    Update $\theta_\ell$ with local gradient using $\partial \mathcal{L}/\partial \theta_\ell$\;
    \If{$\ell = L$}{Update $\theta_h$\;}
    \textbf{delete} $z_{\ell-1}$\;
}
\end{algorithm}

\section{Proof for Proposition \ref{thm-early-training}}
\label{appendix:proof}
We use the notation of the main text: the network is partitioned into blocks $z_\ell = f_{\theta}^\ell(z_{\ell-1})$, each block is residual $f_{\theta}^\ell(z) = z + r_{\theta}^\ell(z)$, and the proxy is residual $p_{\phi}^\ell(z) = z + q_{\phi}^\ell(z)$. We assume that every residual above $\ell$, that is $r_{\theta}^{\ell'}$ for $\ell' > \ell$, and the proxy residual $q_{\phi}^\ell$ (linear or not) are all $\epsilon$-Lipschitz, i.e. $\| r_{\theta}^{\ell'}(z) - r_{\theta}^{\ell'}(z') \| \leq \epsilon \| z - z' \|$ and $\| q_{\phi}^\ell(z) - q_{\phi}^\ell(z') \| \leq \epsilon \| z - z' \|$ for all $z, z'$. We write $R^{\ell'} = \partial r_{\theta}^{\ell'}/\partial z$ and $Q^\ell = \partial q_{\phi}^\ell/\partial z$ for the Jacobians of these residual maps, so that $\| R^{\ell'} \|_{\mathrm{op}} \leq \epsilon$ and $\| Q^\ell \|_{\mathrm{op}} \leq \epsilon$ everywhere ($\| \cdot  \|_{\mathrm{op}}$ denotes the operator norm). By definition of the backpropagated gradient and of ProxyProp's gradient estimate in Equation~\ref{eq:proxygrad}, we recall that:

\begin{equation}
g_{\theta_\ell}^{BP} = \frac{\partial f_{\theta}^\ell}{\partial \theta}^{\!\top} \frac{\partial z_L}{\partial z_\ell}^{\!\top} \frac{\partial \mathcal{L}}{\partial z_L}\,,
\qquad
g_{\theta_\ell}^{loc} = \frac{\partial f_{\theta}^\ell}{\partial \theta}^{\!\top} \frac{\partial p_{\phi}^\ell}{\partial z_\ell}^{\!\top} \frac{\partial \mathcal{L}}{\partial z_L}\,.
\end{equation}

\textbf{Expanding the stacked Jacobian.} By the chain rule, $\frac{\partial z_L}{\partial z_\ell} = \prod_{\ell'=\ell+1}^{L} (I + R^{\ell'})$, ordered with decreasing index from left to right, a product of $L-\ell$ near-identity factors, one for each block above $\ell$. Expanding this product as a sum over which factors contribute $R^{\ell'}$ and which contribute $I$ gives $2^{L-\ell}$ paths inside the computational graph (at each layer a path either goes through a skip connection or through a residual map). The leading path $\pi_0$, which goes only through identities, contributes exactly $I$, and is what the naive identity proxy ($p_{\phi}^\ell(z_\ell) = z_\ell$) estimates. Grouping the remaining paths by their number of non-identity factors:

\begin{equation}
\frac{\partial z_L}{\partial z_\ell} = I + {\sum_{\ell'=\ell+1}^{L} R^{\ell'}} + {\sum_{\ell < \ell' < \ell'' \leq L} R^{\ell''} R^{\ell'} + \dots}\,.
\end{equation}

Since $\| R^{\ell'} \|_{\mathrm{op}} \leq \epsilon$ for all $\ell'$, the sum of the $\binom{L-\ell}{k}$ terms with $k \geq 2$ paths is bounded, by submultiplicativity of the operator norm, by $\sum_{k \geq 2} \binom{L-\ell}{k} \epsilon^k = O(\epsilon^2)$, so:

\begin{equation}
\frac{\partial z_L}{\partial z_\ell} = I + \sum_{\ell'=\ell+1}^{L} R^{\ell'} + O(\epsilon^2)\,.
\end{equation}

Likewise, by definition of the residual proxy, $\frac{\partial p_{\phi}^\ell}{\partial z_\ell} = I + Q^\ell$ exactly, with $\| Q^\ell \|_{\mathrm{op}} \leq \epsilon$.

\textbf{Comparing ProxyProp to backpropagation.} Subtracting the two expansions:

\begin{equation}
\frac{\partial z_L}{\partial z_\ell} - \frac{\partial p_{\phi}^\ell}{\partial z_\ell} = \sum_{\ell'=\ell+1}^{L} R^{\ell'} \; - \; Q^\ell \; + \; O(\epsilon^2)\,,
\end{equation}

so that, using the triangle inequality and $\| R^{\ell'} \|_{\mathrm{op}}, \| Q^\ell \|_{\mathrm{op}} \leq \epsilon$:

\begin{equation}
\Big\| \frac{\partial z_L}{\partial z_\ell} - \frac{\partial p_{\phi}^\ell}{\partial z_\ell} \Big\|_{\mathrm{op}} \leq (L-\ell)\, \epsilon + \epsilon + O(\epsilon^2) = (L-\ell+1)\, \epsilon + O(\epsilon^2)\,.
\end{equation}

Finally, using submultiplicativity of the operator norm, and setting $C = (L-\ell+1)\, \| \partial f_{\theta}^\ell/\partial \theta \|_{\mathrm{op}} \, \| \partial \mathcal{L}/\partial z_L \|$:

\begin{equation}
\| g_{\theta_\ell}^{BP} - g_{\theta_\ell}^{loc} \| \leq
\Big\| \frac{\partial f_{\theta}^\ell}{\partial \theta} \Big\|_{\mathrm{op}}
\Big\| \frac{\partial z_L}{\partial z_\ell} - \frac{\partial p_{\phi}^\ell}{\partial z_\ell} \Big\|_{\mathrm{op}}
\Big\| \frac{\partial \mathcal{L}}{\partial z_L} \Big\|
\leq C\, \epsilon + O(\epsilon^2)\,,
\end{equation}

which is the claimed bound. The argument never used linearity of $q_{\phi}^\ell$: only $\| Q^\ell \|_{\mathrm{op}} \leq \epsilon$ was needed, so it applies equally to non-linear proxies. The naive identity proxy ($q_{\phi}^\ell = 0$, so $Q^\ell = 0$) is the valid special case.

\section{Experiment Details}
\label{appendix:detail}

\subsection{Numerical Verification}

A 6-layer MLP on MNIST with width 64 and batch size 4096.
Each layer is a transformer-style branch $W_{l,2}\textnormal{ReLU}(W_{l,1},z_{l-1})$ with hidden width 256 either with or without the skip connection.

Figure \ref{subfig:mlpgradalign} solves a linear proxy closed-form to $z_L$ either on the plain network or to $z_L - z_l$ on the residual network.
The alignment is measured between the local gradient at each layer through the proxy and the true gradient.

Figure \ref{subfig:residpred} freezes the residual network at random and measures the alignment of layer 0 when the linear proxy is trained by Adam for 800 steps.
The flat line denotes the closed-form solve.

\subsection{ViT}

\subsubsection{Specifications}

All CIFAR-100 experiments use a 12-layer vision transformer of width 128, 4 attention heads, patch size 4 with classification on a final mean token readout.
For local methods, the final LayerNorm is placed in the classifier rather than in the last block so that local methods read raw activations at each partition boundary.
The classifier is zero-initialized, as in the original ViT.

Training uses the AdamW optimizer with batch size 128 at learning rate $5\times10^{-4}$ with 10\% linear warm-up and cosine decay over 100 epochs.
ProxyProp uses a single transformer at width 64 and is trained by the same optimizer as the backbone.
We report last-epoch accuracies.

To highlight the savings in memory in Figure~\autoref{subfig1b:accmem}, we increase the batch size to 512 and measure peak memory usage for one epoch.

\subsection{Finetuning}
We follow the training setup in MeZO \citep{malladi2023mezo} and build our experiments based on their published repository. All models are finetuned for 10 epochs with batch size 8, and the epoch with the best validation loss is used for evaluation on the test dataset. For OPT-13B, we use rank-8 LoRA and split the 40-layer network into 8 blocks. 
We use the Adam optimizer and swept the learning rate ($3\times 10^{-5}, 5\times 10^{-5}, 10^{-4}, 2\times 10^{-4}$) to select the optimal learning rate. 
The proxies are parameterized as: a linear projection down to 80 dimension, a small OPT transformer layer with 80 hidden dimension and 40 heads, and a linear projection back to network hidden dimension. The second linear projector is zero-initialized. A residual connection to each proxy is added unless specified. For LLaMA-3B, we split the 28 layer network into 7 blocks. The proxies are parameterized as a small LLaMA transformer layer with 64 hidden dimensions and linear projectors mapping to network hidden dimension. We use the Adam optimizer and swept the learning rate ($5\times 10^{-6}, 10^{-5}, 3\times 10^{-5}$) to select the optimal learning rate. For all experiments, parameters of the proxies are trained with the same learning rate as the network parameters. For the local loss objective, we use the shared classification head for all blocks, because separate classification heads would result in a large memory cost. 
We use bfloat16 mixed precision training. All other settings remain exactly the same as in \cite{malladi2023mezo}. Experiments are performed on one 80G Nvidia H100 for SQuAD and one A100 for SST2 and WIC.

When profiling the memory, we train the network with synthetic random data. To reduce the large memory requirement of language head, we compute the cross-entropy loss by chunking the final activation along the sequence dimension into pieces of a length 512.



\section{Additional ablations}
\label{appendix:ablation}

\subsection{ViT experiments}

We perform further ablations in Table \ref{tab:b3-baselines-ext}.
At convergence, we find no significance between residual and non-residual parameterization of the proxy in the classification, despite the clear advantage of residual in fine-tuning tasks.
Further, we find that by prepending some transformer layers to the head \citep{gelbard2026unlocking} and training the proxies to distill the input into the deep head is advantageous. In this case, $z_L$ would not be the last layer of the network but the last $k$-th layer with the parameters of the last $k$ layers as part of $\theta_h$. We found that it could further improve performance at roughly the same memory cost.

To make ProxyProp ``update unlocked,'' we use Minitron \citep{muralidharan2024minitron} to prune the transformer layer above each block into a 64-wide transformer proxy at the beginning of each epoch.
We train this proxy for 20\% of the batch, with the backbone frozen.
Then, we freeze the proxy, and train the backbone with the proxy frozen for the entire batch. 

Finally, we offer another baseline where we train a single block of the partitioned network with the proxy attached with end-to-end by backpropagation on the classification loss under the same optimizer. 
In ProxyProp, the proxy represents the distilled downstream network; this baseline suggests this distillation is more useful than simply augmenting the block-wise local loss with an additional proxy.

\begin{table}
    \centering
    \small
    \caption{\textbf{Extended ablations and baselines of ViT on CIFAR-100.}}
    \label{tab:b3-baselines-ext}
    \begin{tabular}{lcccc}
    backprop & & $66.13 {\scriptstyle{\pm 0.55}}$ & & 12.30 \\
    \hline
    Method & & top-1 accuracy & & RAM (GiB)\\
     & {$\scriptstyle 3\times4$} & {$\scriptstyle 6\times2$} & {$\scriptstyle 12\times1$} & {$\scriptstyle 3\times4/6\times2/12\times1$} \\
    \hline
    ProxyProp & ${64.90} {\scriptstyle \pm 0.44}$ & ${61.51} {\scriptstyle \pm 0.19}$ & ${57.58} {\scriptstyle {\scriptstyle \pm 0.40}}$ & 4.64 / 2.76 / 1.90 \\
    \quad non-residual proxy & $64.86 {\scriptstyle \pm 0.43}$ & $61.46 {\scriptstyle \pm 0.20}$ & $57.73 {\scriptstyle \pm 0.58}$ & 4.62 / 2.74 / 1.88 \\
    \quad deep head & $65.47 {\scriptstyle \pm 0.77}$ & $63.38 {\scriptstyle \pm 0.47}$ & --- & 4.85 / 2.92 / --- \\
    \quad fully-unlocked & $64.78 {\scriptstyle \pm 0.42}$ & $59.79 {\scriptstyle \pm 0.20}$ & $42.02{\scriptstyle \pm 1.28}$ & 4.60 / 2.67 / 1.71 \\
    \quad closed-form solve (EMA) & $62.93 {\scriptstyle \pm 0.59}$ & $58.95 {\scriptstyle \pm 0.47}$ & $48.16 {\scriptstyle \pm 0.43}$ & 4.62 / 2.69 / 1.72 \\
    \quad closed-form solve & $58.56 {\scriptstyle \pm 1.42}$ & $46.40 {\scriptstyle \pm 7.81}$ & $31.92 {\scriptstyle \pm 0.84}$ & 4.60 / 2.67 / 1.70 \\
    \quad random residual linear proxy & $57.37 {\scriptstyle \pm 0.63}$ & $55.90 {\scriptstyle \pm 0.39}$ & $46.03 {\scriptstyle \pm 0.25}$ & 4.62 / 2.69 / 1.72 \\
    \quad random linear proxy & $21.03 {\scriptstyle \pm 9.49}$ & $13.02 {\scriptstyle \pm 1.19}$ & $13.87 {\scriptstyle \pm 1.71}$ & 4.60 / 2.67 / 1.70 \\

    \hline
    backprop (one block) & $63.93 {\scriptstyle \pm 0.60}$ & $60.00 {\scriptstyle \pm 0.54}$ & $50.33 {\scriptstyle \pm 0.32}$ & 4.44 / 2.47 / 1.49 \\
    BP (one block + proxy)& $63.71 {\scriptstyle \pm 0.46}$ & $60.31 {\scriptstyle \pm 0.26}$ & $54.39 {\scriptstyle \pm 0.44}$ & 4.44 / 2.47 / 1.49 \\
    
    \end{tabular}
  \end{table}
  
\subsection{Finetuning}

We also study the importance of proxy training in SQuAD fine-tuning tasks. We found that the identity proxy has already reached a good performance, but a trained proxy further improves the performance. Consistent with findings on ViT, prepending the last block with the loss head improves the performance.
\begin{table}
  \caption{Performance of applying ProxyProp variants to finetuning OPT-13B (with LoRA-8) on SQuAD. All ProxyProp variants split the network into 8 blocks and use proxies with hidden dimension 80. The value after $\pm$ is the 95\% confidence interval computed from 5 different random seeds.}
  \label{tab: finetune-ablation}
  \begin{center}
    \begin{small}
          \begin{tabular}{lc}
           \hline
            Method & SQuAD Performance\\
            \hline
            Backprop & $88.30 \pm 0.82$ \\
            ProxyProp & $84.75 \pm 0.91$\\
            ProxyProp with fixed identity proxy & $82.15 \pm 1.82$\\
            ProxyProp + last block in the head& \textbf{$86.20 \pm 1.15$} \\
            \hline
        \end{tabular}
    \end{small}
  \end{center}
\end{table}

\end{document}